\documentclass{article}

\usepackage[utf8]{inputenc} 
\usepackage[T1]{fontenc}    
\usepackage{hyperref}       
\usepackage{url,wrapfig}            
\usepackage{booktabs}       
\usepackage{amsfonts}       
\usepackage{nicefrac}       
\usepackage{microtype}      
\usepackage{amsmath,amssymb,amsthm, mathtools}
\usepackage{subfigure}
\usepackage{geometry}
\usepackage[final]{graphicx}
\usepackage{xcolor}
\usepackage{natbib}
\usepackage{mathrsfs}
\usepackage{algorithm, algorithmic}
\renewcommand{\cal}{\mathcal}

\newcommand{\cX}{{\cal X}}

\newcommand{\cE}{{\cal E}}

\newcommand{\cN}{{\cal N}}

\newcommand{\cP}{{\cal P}}

\newcommand{\cR}{{\mathcal R}}

\newcommand{\cF}{{\mathcal{ F}}}

\newcommand{\bE}{\mathbb{E}}

\newcommand{\bN}{\mathbb{N}}
\newcommand{\bP}{\mathbb{P}}

\newcommand{\bR}{{\mathbb R}}
\newcommand{\bS}{\mathbb S}

\newtheorem{theorem}{Theorem}
\newtheorem{proposition}{Proposition}
\newtheorem{lemma}{Lemma}
\newtheorem{corollary}{Corollary}
\newtheorem{remark}{Remark}
\newtheorem{assumption}{Assumption}

\newcommand{\R}{\mathbb R}
\newcommand{\E}{\mathbb E}
\newcommand{\Prob}{\mathbb P}
\newcommand{\sgn}{{\rm sign}}

\DeclareMathOperator{\argmax}{argmax}
\DeclareMathOperator{\argmin}{argmin}

\def\bO{\mathbb{O}}

\title{Adversarial Training Helps Transfer Learning via Better Representations }

\author{Zhun Deng$^{1}$\footnote{Equal contribution.}
	\quad Linjun Zhang$^{2*}$
	\quad Kailas Vodrahalli$^{3}$
	\quad Kenji Kawaguchi $^{1}$
	\quad James Zou$^{4}$}

\begin{document}
	\date{}
	
	\maketitle
	\footnotetext[1]{Harvard University, \emph{zhundeng@g.harvard.edu,~kkawaguchi@fas.harvard.edu}}
	\footnotetext[2]{Rutgers University, \emph{linjun.zhang@rutgers.edu}}
	\footnotetext[3]{Stanford University, \emph{kailasv@stanford.edu}}
	\footnotetext[4]{Stanford University, \emph{jamesz@stanford.edu}}

\begin{abstract}
Transfer learning aims to leverage models pre-trained on source data to efficiently adapt to target setting, where only limited data are available for model fine-tuning. Recent works empirically demonstrate that adversarial training in the source data can improve the ability of models to transfer to new domains. However, why this happens is not known. In this paper, we provide a theoretical model to rigorously analyze how adversarial training helps transfer learning. We show that adversarial training in the source data generates provably better representations, so fine-tuning on top of this representation leads to a more accurate predictor of the target data.  
 We further demonstrate both theoretically and empirically that semi-supervised learning in the source data can also improve transfer learning by similarly improving the representation. Moreover, performing adversarial training on top of semi-supervised learning can further improve transferability, suggesting that the two approaches have complementary benefits on representations.  We support our theories with experiments on popular data sets and deep learning architectures. 
\end{abstract}

\section{Introduction}

Transfer learning is a popular methodology to obtain well-performing machine learning models in settings where high-quality labeled data is scarce \cite{donahue2014decaf,sharif2014cnn}.
The general idea of transfer learning to take a pre-trained model from a source domain---where labeled data is abundant---and adapt it to a new target domain. Because the target data distribution often differs from the source setting, standard transfer learning fine-tunes the model using a small-amount of labeled data from the target domain. In many applications, the fine-tuning is performed only on the last few layers of the network if the amount of target data is limited or if the one only has access to a representation (\emph{i.e.} intermediate layers) produced by the model instead of the full model.  

Transfer learning has demonstrated substantial empirical success and there is an exciting literature investigating different approaches to making transfer learning more effective \cite{huh2016makes,kolesnikov2019big}. Recent experiments empirically demonstrated an intriguing phenomenon that models that are trained using adversarial-robust optimization on the source data transfer better to target data compared to non-adversarially trained models.   
We illustrate this phenomenon in Figure \ref{fig:intro}, which replicates the findings in \cite{salman2020adversarially}. Here two models are trained on the full ImageNet 
and 10\% of ImageNet using different levels of adversarial training---$\epsilon$ is the $l_2$ magnitude of the adversarial attack. Following \cite{salman2020adversarially}, we fine-tuned the last layer of the models using data from CIFAR-10 and plot the final accuracy on the target CIFAR-10. Adversarial training ($\varepsilon>0$) significantly improves the transfer performance compared to model without adversarial training ($\varepsilon=0$). Additional experiments demonstrating this effect are provided in \cite{salman2020adversarially,utreraadversarially}, however it is still an open question how adversarial training in source helps transfer learning.    


\begin{wrapfigure}{r}{0.4\textwidth}
\centering
	\includegraphics[width=0.4\textwidth]{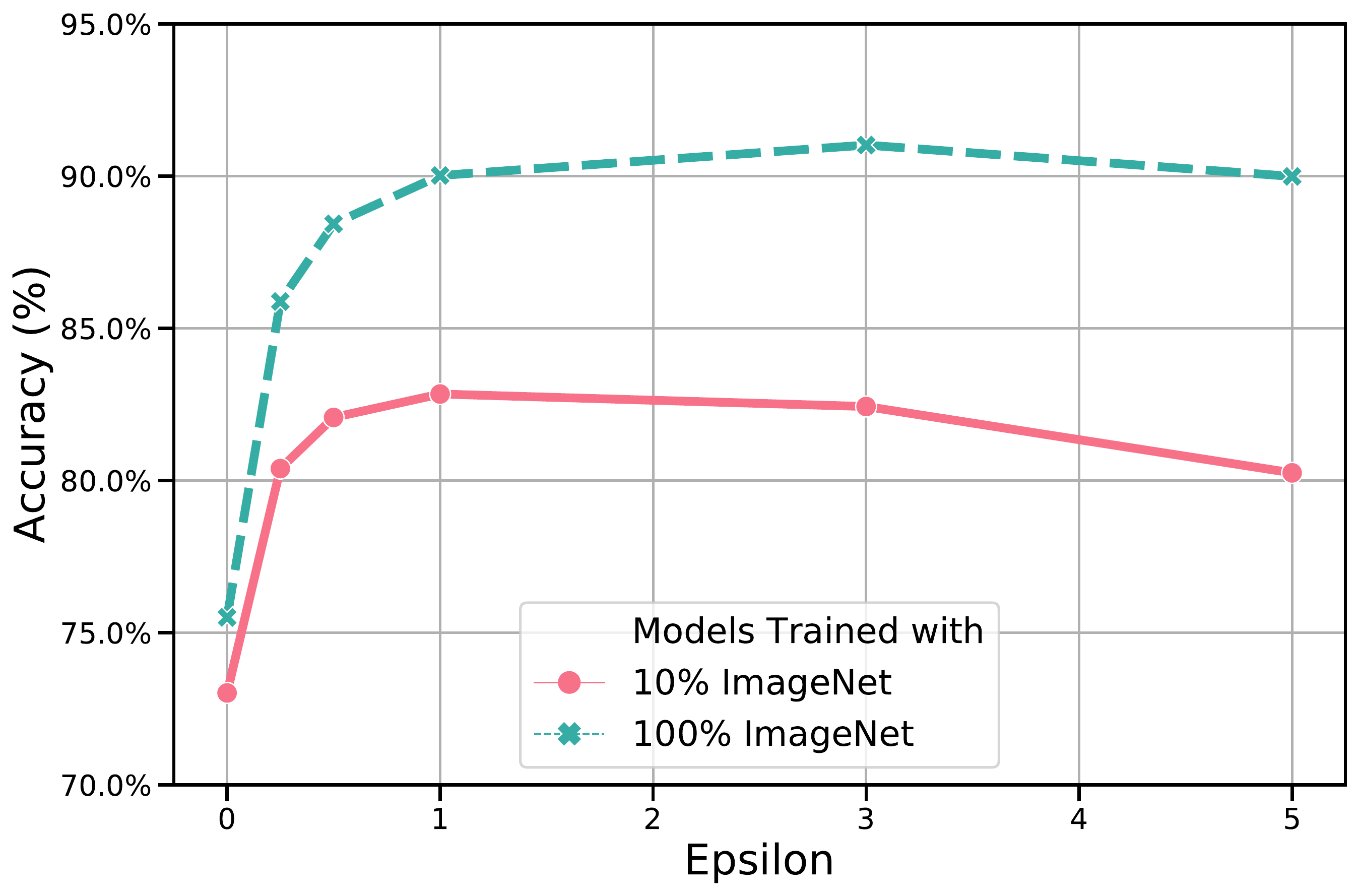}
	\caption{Transfer accuracy improves with adversarial training on source task. We plot target task (CIFAR-10) accuracy across different levels of $\ell_{2}$-adversarial training on the source task (ImageNet). The value of $\varepsilon$ corresponds to the size of the adversarial attack; \emph{i.e.}, $\varepsilon = 0$ indicates no adversarial training. The two curves correspond to training the source model using all of ImageNet and a 10\% subsample of ImageNet.}
	\label{fig:intro}
	\vspace{-1em}
\end{wrapfigure}

As our \textbf{first contribution}, we initialize the study of how adversarial training helps fixed-feature transfer learning from a theoretical perspective.
Our analysis shows how that adversarial training on the source learns a better representation such that fine-tuning on this representation leads to better performance on the target. Interestingly, we show that the robust representation can help transfer learning even when the source performance declines due to adversarial training. To the best of our knowledge, this is the first rigorous analysis of the effect of adversarial training on transfer learning.  

As our \textbf{second contribution}, we extend our analysis to show that semi-supervised learning using pseudo-labeling can similarly lead to better representations for transfer learning. We support our theory with empirical experiments. Moreover our experiments demonstrate for the first time that performing adversarial training on top of pseudo-labeling in the source can further boost transfer learning performance. This suggests that the two data augmentation techniques of adversarial training and pseudo-labeling have complementary benefits on learned representations. 

As a \textbf{third technical contribution}, we generalize the techniques in prior papers for analyzing transfer learning in regressions to classification settings, where adversarial training and pseudo-labeling are more commonly used. Together, 
our results provide a useful and tractable framework to understand factors that improve transfer learning.

\paragraph{Related Work} 
Adversarial robust optimization has been a major focus in machine learning security \citep{biggio2018wild, dalvi2004adversarial, lowd2005adversarial,goodfellow2014explaining,carlini2017towards,nguyen2015deep}. A serious of works has been proposed to increase the adversarial robustness both empirically \cite{madry2017towards,miyato2018virtual,balaji2019instance} and theoretically \cite{cohen2019certified,lecuyer2019certified,raghunathan2018certified,liu2020enhancing,chiang2020certified,kurakin2016adversarial,hendrycks2019using,deng2020towards,zhang2020does}.  Meanwhile, other works demonstrate how to quantify the trade-off of adversarial robustness and standard accuracy \citep{zhang2019theoretically, schmidt2018adversarially,carmon2019unlabeled,stanforth2019labels,deng2020interpreting}. Recently \cite{utrera2020adversarially, salman2020adversarially} empirically studied the transfer performance of adversarially robust networks, but it is still not clear yet why adversarial training leads to a better transfer from a theoretical perspective.

Transfer learning has been used in a variety of applications, ranging from medical imaging \cite{raghu2019transfusion}, natural language processing \cite{houlsby2019parameter,conneau2018senteval}, to object detection \cite{lim2012transfer,shin2016deep}. 
On the theoretical side,  the early work of \cite{baxter2000model,ben2003exploiting,maurer2016benefit} studied the test accuracy of the target task in the multi-task learning setting. Recent work \cite{tripuraneni2020theory, tripuraneni2020provable, du2020few} focused more on the representation learning and provide a theoretical framework to study linear representatioin in the \textit{regression} setting. In this work, we provide a counterpart to theirs and studies the \textit{classification} setting. 
Prior works in semi-supervised leaerning largely focus on improving the prediction accuracy with unlabeled data \citep{zhu2003semi,zhu2009introduction,berthelot2019mixmatch}. Works have also shown that semi-supervised learning can improve adversarial robustness \citep{carmon2019unlabeled,deng2020improving}. Several works have identified that using unlabeled data can empirically improve transfer learning \citep{zhou2018semi,zhong2020multispecies,mokrii2021systematic}, but a rigorous theoretical understanding of why this happens is lacking.

\section{Preliminaries and model setup}\label{sec:2}
\paragraph{Notation.}
We use $[m]$ for $\{1,2,\cdots,m\}$ for any $m\in\bN^+$ and for any set $S$, let $|S|$ to denote the cardinality of $S$. For a matrix $M$, we denote $\sigma_k(M)$ as the $k$-th singular value of matrix $M$. We use $\bO_{m\times l}$ to denote the space of matrices of dimension $m\times l$ whose columns are orthonormal and use $\bS^{p-1}$ to denote the unit sphere of dimension $p$. For two real matrices $E, F\in \bO_{m\times l}$, we denote the subspaces spanned by the column vectors of $E$ and $F$ by $\cE$ and $\cF$ correspondingly. The subspace distance between $\cE$ and $\cF$ is defined as $\|\sin\Theta(E,F)\|_F$ \citep{yu2015useful}, where $\Theta(E,F)=\text{diag}(\cos^{-1}\sigma_1(E^\top F),\cdots, \cos^{-1}\sigma_l(E^\top F))$. For a vector $v$, we use $\|v\|_q$ to denote the $\ell_q$ norm. Let $\lesssim$ and $\gtrsim$ denote ``less than'' and ``greater than'' up to a universal constant respectively. $a \ll b$ to denote $b\ge C\cdot a$ for a sufficiently large universal constant $C$. Our use of $O(\cdot),\Omega(\cdot), o(\cdot)$ follows the standard literature of computer science. With some abuse of notation, we also write $a=\Theta(b)$ if $a=O(b)$ and $a=\Omega(b)$ for $a,b\in \bR$. 
\paragraph{Data generating processes.} We assume there are $T$ source tasks. For each task $t\in [T]$, we have corresponding training data set of size $n_t$, \emph{i.e.} $S_t=\{(x^{(t)}_1,y^{(t)}_1),\cdots,(x^{(t)}_{n_t},y^{(t)}_{n_t})\}$, where $x^{(t)}_i\in\cX\subseteq\bR^p$ and $y^{(t)}_i\in\{-1,1\}$ are i.i.d. drawn from a joint distribution $\cP^{(t)}_{x,y}$. We further denote $n=\min_{t\in[T]}n_t$. In other words, $n$ is the smallest size of source data sets.
The goal of transfer learning is to learn from multiple source tasks in the hope of learning a common representation such that for a target task with distribution $\cP^{(T+1)}_{x,y}$, we only need few data points to learn extra structures beyond the common representation and the learned model still achieves good prediction performance. With this spirit, we assume that for $t\in[T+1]$, $\{(x^{(t)}_i,y^{(t)}_i)\}_i^{n_t}$ are i.i.d. drawn from $\cP^{(t)}_{x,y}$, such that
\begin{equation}\label{eq:model}
x^{(t)}_i=\eta^{(t)}_i+y^{(t)}_i\mu_t,
\end{equation}
for i.i.d. noise $\eta^{(t)}_i$ that is independent of $y^{(t)}_i$, where $\mu_t=Ba_t\in\R^p$, $a_t\in\R^r$ and $B\in\R^{p\times r}$ is an orthonormal matrix representing the projection onto a subspace, \emph{i.e.} $B^\top B=I_r$. Here, $B$ is the common structure shared among all the source tasks and the target task, and $a_t$'s are task-specific parameters. Although this model is simple, the analysis is already highly nontrivial, and it captures the essense of the problem in transfer learning. In fact, similar models haves been considered in \cite{tripuraneni2020provable, du2020few}. Specifically, we consider the case where $r\ll p$. It can be viewed in a way that the data is generated by mapping low dimensional data signal to the high-dimension, which coincides with the fact that commonly used real image data sets lie in the lower dimensional manifolds. 
In addition, we assume the noise term $\eta^{(t)}_i$ is of zero-mean and is $\rho_t^2$-sub-gaussian, \emph{i.e.} $\bE[\exp(\lambda v^\top \eta^{(t)}_i)]\le \exp(\lambda^2\rho_t^2/2)$ for all $v\in\bS^{p-1}$ and $\lambda \in\bR$. Throughout this paper, we consider $\rho_t=\Theta(1)$ for all $t\in[T]$.
\begin{remark}
(i). The sub-gaussian assumption is quite flexible since many commonly used data sets such as image sets are all bounded, which implies sub-gaussianity. (ii). Different from the regression settings considered in previous theoretical work on transfer learning\cite{du2020few,tripuraneni2020provable}, we focus on \textbf{classification settings}, in which adversarial training is more commonly studied. 
\end{remark}

\paragraph{Loss functions.} The loss functions considered in this paper take the following form: for each task $t\in [T+1]$,
\begin{equation}\label{eq:loss}
\ell(x,y,w^{(t)}_2, W_1)=-yf^{(t)}(x),
\end{equation}
where $f^{(t)}(x)$ is a two-layer linear neural network parametrized by $W_1$ and $w^{(t)}_2$, \emph{i.e.}
$f^{(t)}(x)=w^{(t)\top}_2 W_1^\top x,$ with $W_1\in \bO_{p\times r}$, $w^{(t)}_2\in \bR^{r\times 1}$ and $\|w^{(t)}_2\|\le 1$. Here, we mainly consider the case $W_1$ is well-specified, \emph{i.e.} with the same dimension of $B$. Our argument can be further extended to the case where $r$ is unknown by first estimating $r$ and details are left to the appendix. We put norm constraint on $w^{(t)}_2$ since otherwise the minimizer is always of norm infinity. The loss function in \eqref{eq:loss} along with its variants have been commonly used in the theoretical machine learning community \cite{schmidt2018adversarially,deng2020improving}. Although in its simple form, it has been consistently useful to shed light upon complex phenomena. Meanwhile, even under this natural setting, it is highly non-trivial to demonstrate the effect of adversarial training in transfer learning.

Roughly speaking, like most settings in transfer learning, $W_1$ is assumed to be the common weights shared among the models for all source tasks so as to learn a ``good" common representation. For each individual task $t$, parameter $w^{(t)}_2$ aims to perform task-specific linear classification. We leave the detailed discussions about how to take advantage of combining all the source tasks and obtaining a ``good" $W_1$ to Section~\ref{sec:mainres}. Further, we denote the empirical loss for task $t$ as
$\hat{L}(S_t,w^{(t)}_2,W_1)=\sum_{i=1}^{n_t} -y^{(t)}_i\langle W_1w^{(t)}_2, x^{(t)}_i \rangle/n_t.$
The expected loss for task $t$ is
$L(\cP_{x.y}^{(t)},w^{(t)}_2,W_1)=-\bE_{(x,y)\sim \cP_{x,y}^{(t)}}[ y\langle W_1w^{(t)}_2, x \rangle].$

\paragraph{Problem Setup}
In \emph{fixed-representation} transfer learning, the first step is to learn the common representation in the model architectures using data from source tasks. The representation (e.g. the penultimate layer of a neural network) is then fixed. Finally, the target data is used to train or fine-tune a small model on top of the representation. 
Following this popular practice, in our model setting, we use the data of $T$ source tasks $\{S_t\}_{t=1}^T$ to obtain an estimator $\hat{W}_1$. Then, we use the data of target task $S_{T+1}$ to obtain an estimator $\hat{w}^{(T+1)}_2$ of the task-specific parameter. Our evaluation criteria is the \emph{excess risk}:
\begin{equation}
\cR(\hat{W}_1,\hat{w}^{(T+1)}_2)=L(\cP_{x,y}^{(T+1)},\hat{w}^{(T+1)}_2,\hat{W}_1)-\min_{\|w_2\|\le 1,W_1\in \bO_{p\times r}}L(\cP_{x,y}^{(T+1)},w_2,W_1).
\end{equation}

\section{Adversarial Training Help Representation Learning}\label{sec:mainres}
In this section, we demonstrate our results about how adversarial training can learn a better representation, and therefore leads to smaller excess risks.  We first describe our algorithm, and 
demonstrate the near-optimality of our algorithm in representation learning by a minimax lower bound. 
We then demonstrate for the settings where data has varying noise-signal ratios or sparsity structures, how $\ell_2$ or $\ell_\infty$-adversarial training can help improve the representation learning. 

\subsection{Representation learning algorithm}

Recall that the loss function for each task is $\ell(x,y,w^{(t)}_2, W_1)=-yw^{(t)\top}_2 W_1^\top x$, where $W_1\in \bO_{p\times r}$ is a common structure in model architectures shared among all the source tasks and the target set. In the spirit of transfer learning, the goal is to jointly learn $W_1$ from source tasks and then use the data from the target task to learn its task-specific parameter $w^{(T+1)}_2$. Here, $W_1$ essentially aims to recover the common structure $B$ in the data generating processes Eq.\eqref{eq:model} (or more rigorously, recover the column space of $B$), such that the obtained estimator $\hat{W}_1$ satisfies $\|\sin\Theta(\hat{W}_1,B)\|_F \to 0$. 

Note that in our two-layer linear neural network structure, optimizing $w_2$ and $W_1$ simultaneously for a single task has the issue of non-identifiability -- the loss value will not change if we multiply an orthonormal matrix $\Lambda\in\R^{r\times r}$ to $W_1$ and $\Lambda^{-1}$ to $w_2$. 
However, we still can jointly learn a good estimator $\hat{W}_1$ to recover $B$ following a similar method in \cite{tripuraneni2020provable} via singular value decomposition (SVD). In particular, we first simultaneously optimize $w^{(t)}_2$ and $W_1$ for each individual task for $t\in[T]$, which is equivalent to optimizing a single parameter $\beta_t=W_1^\top w^{(t)}_2$ (since $W_1$ is an orthonormal matrix, the norm of $\beta_t$ is still upper bounded by $1$). Then, we apply SVD to the matrix consisting of the optimizers $\hat{\beta}_t$'s to obtain $\hat{W} _1$. In the final step, we use $S_{t+1}$ to learn $w_2^{(t+1)}$. 
\begin{algorithm}[tbh]
   \caption{Learning for Linear Representations}
   \label{alg:natural}
\textbf{Input:} $\{S_t\}_{t=1}^{T+1}$
\vspace{0.1cm}

\text{Step 1:} Optimize the loss function on each individual source task $t\in[T]$ and obtain
$$\hat{\beta}_t =\argmin_{\|\beta_t\|\le 1} \frac{1}{n_t}\sum_{i=1}^{n_t}-y^{(t)}_i\langle \beta_t, x^{(t)}_i \rangle. $$

\text{Step 2:} $\hat{W}_1 \Sigma V^\top\leftarrow \text{top-$r$}~\text{SVD of}~[\hat{\beta}_1,\hat{\beta}_2,\cdots,\hat{\beta}_T],$ where $\Sigma$ is  a $r\times r$ diagonal matrix consists of singular values, and $V$ is a $T\times r$ matrix consists of orthonomal columns.

\text{Step 3:} $\hat{w}^{(T+1)}_2\leftarrow \argmin_{\|w^{(T+1)}_2\|\le 1} \frac{1}{n_{T+1}}\sum_{i=1}^{n_{T+1}}-y^{(T+1)}_i\langle w^{(T+1)}_2\hat{W}_1, x^{(T+1)}_i \rangle. $

\textbf{Return} $\hat{W}_1$, $\hat{w}^{(T+1)}_2.$
\end{algorithm}

Next, we provide a lemma about the representation learning in the two-layer linear neural network model under the assumption below. 
Combining this lemma with a minimax lower bound, we will show that adversarial training cannot have any gain in representation or transfer learning without extra special data structures, which motivates our subsequent theories. To facilitate the presentation, let us  define $M=[a_1/\|a_1\|,a_2/\|a_2\|,\cdots,a_T/\|a_T\|]$.
\begin{assumption}[Task normalization and diversity]\label{ass:1}
For all the tasks, $\|a_{t}\|=\Theta(1)$ for all $t\in[T+1]$ and $\sigma_r(M^\top M/T)=\Omega(1/r)$. 
\end{assumption}
\begin{remark}
Throughout the paper, we consider the low-rank case, where $r$ is smaller than $T$ and $p$. Meanwhile, notice that $\|M\|_F^2=T=\sum_{i=1}^r\sigma^2_i(M)$, this assumption implies the condition number $\sigma_1(M)/\sigma_r(M)=O(1)$, which roughly means $\{a_i/\|a_i\|\}_{i=1}^T$ cover all the directions of $\bR^r$ evenly.
\end{remark}
Loosely speaking, if we denote $\hat{\mu}_{T+1}=\sum_{i=1}^{n_{T+1}}x^{(T+1)}_iy^{(T+1)}_i/n_{T+1}$, under some regularity conditions, with high probability
\begin{align}
\notag\cR(\hat{W}_1,\hat{w}^{(T+1)}_2)&=L(\cP_{x,y}^{(T+1)},\hat{w}^{(T+1)}_2,\hat{W}_1)-\min_{\|w_2\|\le 1,W_1\in \bO_{p\times r}}L(\cP_{x,y}^{(T+1)},w_2,W_1)\\
&\lesssim \underbrace{\|\sin\Theta(\hat{W}_1,B)\|_F}_{\text{representation error}}+\underbrace{\|B^\top\hat{\mu}_{T+1}-B^\top \mu_{T+1}\|}_{\text{task-specific error}}.\label{eq:excess}
\end{align}
 The task-specific error is easy to deal with given Eq.~\eqref{eq:excess}, we mainly focus on providing a lemma to characterize the representation error.
\begin{lemma}\label{lm:representation}
Under Assumption \ref{ass:1}, if  $n>c_1\max\{pr^2/T,r^2\log(1/\delta)/T,r^2\}$ for some universal constant $c_1>0$ and $2r\le\min\{p,T\}$, for all $t\in [T]$. For $\hat{W}_1$ obtained in Algorithm \ref{alg:natural}, with probability at least $1-O(n^{-100})$,
\begin{equation*}
\|\sin\Theta(\hat{W}_1,B)\|_F\lesssim 
r\left(\sqrt{\frac{1}{n}}+\sqrt{\frac{p}{nT}}+\sqrt{\frac{\log n}{nT}}\right).
\end{equation*}
\end{lemma}
Application of Lemma \ref{lm:representation} gives us the following corollary about the excess risk $\cR(\hat{W}_1,\hat{w}^{(T+1)}_2)$. 

\begin{corollary}\label{col:excessrisk}
Under Assumption \ref{ass:1}, if $n>c_1\max\{pr^2/T,r^2\log(1/\delta)/T,r^2, r n_{T+1}\}$ for some universal constant $c_1>0$, $2r\le\min\{p,T\}$, 
 then for $\hat{W}_1$ obtained in Algorithm \ref{alg:natural}, with probability at least $1-O(n^{-100})$,
\begin{equation*}
\cR(\hat{W}_1,\hat{w}^{(T+1)}_2)\lesssim \sqrt{\frac{r+\log n}{n_{T+1}}}+\sqrt{\frac{r^2p}{nT}}. 
\end{equation*} 
\end{corollary}
\begin{remark}
Lemma \ref{lm:representation} and Corollary \ref{col:excessrisk} provide counterparts of the bound of subspace distance and excess risk studied in \cite{tripuraneni2020provable,du2020few} under the setting of regression models. Since they use squared losses, our bounds are different from theirs by square roots. Squaring our bounds provide results with similar rates as those in previous work. If we do not use data from source tasks, we will obtain an excess risk bound of order $\sqrt{p/n_{T+1}}$ instead, which will be significantly larger than the one in Corollary \ref{col:excessrisk} if $r+\log n\ll p$ and $pr^2\ll n T$, which happens in our low rank situation with abundant source task data.
\end{remark}

 Meanwhile, we provide the following minimax lower bound to justify the near-optimality of our algorithm in learning the representation in general cases.
\begin{proposition}\label{prop:lowerbound}
Let us consider the parameter space $\Xi=\{A\in\R^{p\times r},B\in\R^{p\times r}:\sigma_r(A^\top A/T)\gtrsim 1, B^\top B=I_r\}$. If $nT\gtrsim rp$, we then have
$$
\inf_{\hat W_1}\sup_{\Xi}\E\|\sin\Theta(B,\hat W_1)\|_F\gtrsim
\sqrt\frac{rp}{nT}.
$$
\end{proposition}
\begin{remark}
For high dimensional data such that $p$ is much larger than $T$ and $\log n$, the lower bound in Proposition \ref{prop:lowerbound} matches the upper bound in Lemma \ref{lm:representation} up to a factor $\sqrt{r}$. Since $r$ is considered as a small constant in our settings, we can see that in general cases when there is no additional structural assumptions, 
our algorithm already obtains the near-optimal rate in representation learning. However, in later sections, when we introduce some additional structural assumptions such as varying signal-to-noise ratios and sparsity structures among tasks, which commonly happens in real applications, we will show that adversarial training can improve representation learning and further leads to smaller excess risks.
\end{remark}

\subsection{How $\ell_2$-adversarial training improves representation learning for transfer}
In this subsection, we consider the benefit of $\ell_2$-adversarial training. Specifically, if the signal-to-noise ratios varies among tasks in the sense that  $\|a_t\|$'s have different scales, $\ell_2$-adversarial training can lead to a sharper representation estimation error than standard training. In contrast, Lemma \ref{lm:representation} and Proposition \ref{prop:lowerbound} demonstrate that under the case of uniform signal-to-noise ratios, adversarial training cannot have any gain over standard training. From a high-level perspective, signal-to-noise ratios determiine the difficulties of classification. For those tasks with small signal-to-noise ratios, while adversarial attacks make them even harder to perform classification (increase bias), but also make these tasks less competitive (decrease variance). Thus, adversarial training will bias the model to focus on learning the representation out of those with large signal-to-noise ratios.

\begin{assumption}[Varying signal-to-noise ratios]\label{ass:2}
For the $T$ source tasks, they can be divided into two disjoint sets. The first set is $S_1=\{t\in[T]:\|a_t\|=\Theta(1)\}$, and the second set is $S_2=\{t\in[T]:\|a_t\|=\Omega(\alpha_T)\}$, where $\alpha_T\rightarrow \infty$ as $T\rightarrow\infty$, and $S_1\cup S_2=[T]$.  In addition, $|S_2|/T=\Theta(1)$.
\end{assumption}
For the matrix $M=[a_1/\|a_1\|,a_2/\|a_2\|,\cdots,a_T/\|a_T\|]$, we further denote $M_{S_1}$ as the sub-matrix of $M$, whose columns consist of of $a_t/\|a_t\|$ for $t\in S_1$. For instance, if $S_1=\{1,2,3\}$. then $M_{S_1}=[a_1/\|a_1\|,a_2/\|a_2\|,a_3/\|a_3\|]$. We define $M_{S_2}$ similarly.
\begin{assumption}[Task diversity]\label{ass:3}
For the $T$ source tasks, $\min\{\sigma_r(M_{S_2}^\top M_{S_2}/T),\sigma_r(M^\top M/T)\}=\Omega(1/r)$.
\end{assumption}
\begin{remark}
Assumption \ref{ass:2} indicates if we have more source tasks (larger $T$), more tasks with large signal-to-noise ratios would show up. Similar to Assumption \ref{ass:1}, Assumption \ref{ass:3} requires both the columns in $M$ and $M_{S_2}$ cover $\bR^r$ evenly.
\end{remark}
Now, we consider the adversarial training algorithm for $\ell_q$-attack for $q=2,\infty$.

\begin{algorithm}[tbh]
   \caption{Adversarial Learning for Linear Features}
   \label{alg:adv}
\textbf{Input:} $\{S_t\}_{t=1}^{T+1}$, $q$ 
\vspace{0.1cm}

\text{Step 1:} Optimize the adversarial loss function on each individual source task $t\in[T]$ and obtain
$$\hat{\beta}^{adv}_t =\argmin_{\|\beta_t\|\le 1}\max_{\|\delta_i\|_q\le\varepsilon} \frac{1}{n_t}\sum_{i=1}^{n_t}-y^{(t)}_i\langle \beta_t, x^{(t)}_i+\delta_i \rangle. $$

\text{Step 2:} $\hat{W}^{adv}_1\Sigma^{adv}V^{adv\top}\leftarrow \text{top-$r$}~\text{SVD of}~[\hat{\beta}^{adv}_1,\hat{\beta}^{adv}_2,\cdots,\hat{\beta}^{adv}_T],$ where $\Sigma^{adv}$ is  a $r\times r$ diagonal matrix consists of singular values, and $V^{adv}$ is a $T\times r$ matrix consisting of orthonomal columns.

\text{Step 3:} $\hat{w}^{adv,(T+1)}_2\leftarrow \argmin_{\|w^{(T+1)}_2\|\le 1} \frac{1}{n_{T+1}}\sum_{i=1}^{n_{T+1}}-y^{(T+1)}_i\langle w^{(T+1)}_2\hat{W}^{adv}_1, x^{(T+1)}_i \rangle. $

\textbf{Return} $\hat{W}^{adv}_1$, $\hat{w}^{adv,(T+1)}_2.$
\end{algorithm}


The following theorem shows that even when the $\hat{\beta}^{adv}_t$'s obtained by $\ell_2$-adversarial training have large excess risk for each source task, the $\hat{W}^{adv}_1$ extracted from $[\hat{\beta}^{adv}_1,\hat{\beta}^{adv}_2,\cdots,\hat{\beta}^{adv}_T]$ can transfer knowledge from multiple source tasks better and result in a smaller excess risk on the target task.
\begin{theorem}\label{thm:l2adv}
Under Assumption \ref{ass:2} and \ref{ass:3}, for $\|a_{T+1}\|=\alpha=\Omega(1)$, if $n>c_1\max\{r^2,r/\alpha_T\}\cdot\max\{p\log T,\log n/T,1\}$ and $n>c_2(\alpha\alpha_T)^2r n_{T+1}$  for  universal constants $c_1, c_2$, $2r\le\min\{p,T\}$. There exists a universal constant $c_3$, such that if we choose $\varepsilon \in [\max_{t\in S_1}\|a_t\|+c_3\sqrt{p\log T/n},\min_{t\in S_2}\|a_t\|-c_3\sqrt{p\log T/n}]$ (this set will not be empty if $T, n$ are large enough), for $\hat{W}^{adv}_1$, $\hat{w}^{adv,(T+1)}_2$ obtained in Algorithm \ref{alg:adv} with $q=2$, with probability at least $1-O(n^{-100})$, 
\begin{equation*}
\|\sin\Theta(\hat{W}^{adv}_1,B)\|_F\lesssim (\alpha_T)^{-1}\left(\sqrt{\frac{r^2}{n}}+\sqrt{\frac{pr^2}{nT}}+\sqrt{\frac{r^2\log n}{nT}}\right),
\end{equation*}
and the excess risk 
\begin{equation*}
\cR(\hat{W}^{adv}_1,\hat{w}^{adv,(T+1)}_2)\lesssim \alpha\sqrt{\frac{r+\log n}{n_{T+1}}}+(\alpha_T)^{-1}\left(\sqrt{\frac{r^2p}{nT}}\right).
\end{equation*} 
\end{theorem}

Similar to Assumption \ref{ass:2}, here $\alpha$ can also be a function of the target task data size $n_{T+1}$.

\noindent\textbf{$\ell_2$-adversarial training v.s. standard training:} 
Under the exact same conditions in Theorem~\ref{thm:l2adv}, a simple modification of Lemma~\ref{lm:representation} leads to
 $\|\sin\Theta(\hat{W}_1,B)\|_F\lesssim \sqrt{r^2/n}+\sqrt{pr^2/(nT)}+\sqrt{r^2\log n/(nT)}$ and $
 \cR(\hat{W}_1,\hat{w}^{(T+1)}_2)\lesssim \alpha\sqrt{(r+\log n)/n_{T+1}}+\sqrt{r^2p/(nT)}
 $ with high probability.
We can see that the adversarial training would lead to a better representation and an improved excess risk when $\alpha_T$ is growing. Such a scenario happens when the source data consist of a large diversity of tasks with varying difficulties of classification. Our proof indeed reveals that adversarial training would help the model to focus on learning the representation from easy-to-classify tasks, and therefore improves the convergence rate of representation learning. 
The gain in representation learning further leads to smaller rates of $\cR(\hat{W}^{adv}_1,\hat{w}^{adv,(T+1)}_2)$ compared with $\cR(\hat{W}_1,\hat{w}^{(T+1)}_2)$.

\subsection{How $\ell_\infty$-adversarial training improves representation learning for transfer}
In this subsection, we further consider the benefit of $\ell_\infty$-adversarial training. It is well-recognized that commonly used real data sets, such as MNIST and CIFAR-10, actually lie in lower dimensional manifolds compared with their ambient dimensions. After certain transformations \cite{baraniuk2007compressive,candes2006compressive}, it is equivalent to having sparsity structure in the coordinates. We demonstrate that if there are some underlying sparsity structures in the mean parameters $\mu_{t}=Ba_t$ for $t\in[T]$, then $\ell_\infty$-adversarial training leads to sharper bounds regarding the representation error and excess risk. To facilitate the discussion, let us use $\mu_{t,j}$ to denote the $j$-th coordinates of $\mu_t$. 
\begin{assumption}[Structural sparsity]\label{ass:4}
For an integer $s$ such that $0<s<p$, we assume for all $t\in[T]$, $\sgn(\mu_{t,j}$ are $i.i.d.$ and $\Prob(\sgn(\mu_{t,j})=0)=1-\eta_s$, $\Prob(\sgn(\mu_{t,j})=1)=\Prob(\sgn(\mu_{t,j})=-1)=\eta_s/2$. We also refer $s$ as the sparsity level. 
\end{assumption}
Assumption \ref{ass:4} guarantees that the sparsity of each column is upper bounded by $O(s\log T)$ with high probability. Similar assumptions have been commonly used in the high-dimensional statistics literature \cite{bayati2011lasso,su2017false}.
For $\ell_\infty$-adversarial training, we provide bounds obtained through adversarial training below.

\begin{theorem}\label{thm:linftyadv}
Under Assumptions \ref{ass:1} and \ref{ass:4}, if $n>c_1\cdot r^2\max\{s^2\log^2T/T,r n_{T+1}, 1\}$  for some universal constants $c_1>0$, $2r\le\min\{p,T\}$. There exists a universal constant $c_2$, such that if we choose $\varepsilon>c_2\sqrt{\log p/n}$, for and $\hat{W}^{adv}_1$, $\hat{w}^{adv,(T+1)}_2$ obtained in Algorithm \ref{alg:adv} with $q=\infty$, with probability at least $1-O(n^{-100})-O(T^{-100})$,
\begin{equation*}
\|{\sin\Theta(\hat{W}^{adv}_1,B)\|_F\lesssim r\left(\sqrt{\frac{1}{n}}+\sqrt{\frac{s^2}{nT}}\right)\cdot \log(T+p),}
\end{equation*}
and the excess risk 
\begin{equation}
{\cR(\hat{W}^{adv}_1,\hat{w}^{adv,(T+1)}_2)\lesssim\left(\sqrt{\frac{r+\log n}{n_{T+1}}}+r\sqrt{\frac{s^2}{nT}}\right)\cdot \log(T+p)}.
\end{equation} 
\end{theorem}

\noindent\textbf{$\ell_\infty$-adversarial training v.s. standard training:}
Under the exact same conditions in Theorem~\ref{thm:linftyadv}, again, a simple modification of Lemma~\ref{lm:representation} shows that  without adversarial training, with high probability, we have
$\|\sin\Theta(\hat{W}_1,B)\|_F\lesssim r(\sqrt{{1}/{n}}+\sqrt{{p}/{nT}}+\sqrt{{\log n}/{nT}})$ and the excess risk  $ \cR(\hat{W}_1,\hat{w}^{(T+1)}_2)\lesssim \sqrt{(r+\log n)/n_{T+1}}+r\sqrt{p/nT}$.
 Theorem~\ref{thm:linftyadv}  shows that $\ell_\infty$-adversarial training is able to learn significantly better representations when  $s^2\ll p$. This scenario is common in image classification where the label of an image only depends on a small set of feature. Our proof reveals that $\ell_\infty$-adversarial training would help remove the redundant features in the classification tasks and therefore improves the representation learning and the subsequent downstream prediction on target domain.

\section{Pseudo-Labeling and Adversarial Training}\label{sec:pseudo}
In the previous section, we have shown that combining abundant data from source tasks with robust training can help learn a good classifier for the target task. Sometimes, however, even the sources have limited labeled data. In that case, data augmentation by incorporating unlabeled source data, which are easier to obtain, can be a powerful way to improve prediction accuracy. One of the most commonly used semi-supervised learning algorithms is the pseudo-labeling algorithm \citep{chapelle2009semi}. In this section, we explore how using pseudo-labeling in the source data can improve transfer learning and how adversarial training can further boost that improvement, both empirically and theoretically.

\paragraph{Experiments.} We perform empirical study of image classification. Our source tasks are image classification on ImageNet \cite{russakovsky2015imagenet}; our target tasks are image classification on CIFAR-10 \cite{krizhevsky2009learning}. 
To simulate the pseudo-labeling setup, we sample 10\% of ImageNet, train a ResNet-18 model on this sample (without adversarial training), and generate pseudo-labels for the remaining 90\%. We then train a new source model using all of the source labeled and pseudo-labeled data with and without adversarial training. We use a public library for adversarial training \cite{robustness}. The high-level approach for adversarial training is as follows: at each iteration, take a small number of gradient steps to generate adversarial examples from an input batch; then update network weights using the loss gradients from the adversarial batch. 

\begin{figure}
\begin{center}
\begin{tabular}{cc}
  \includegraphics[width=50mm]{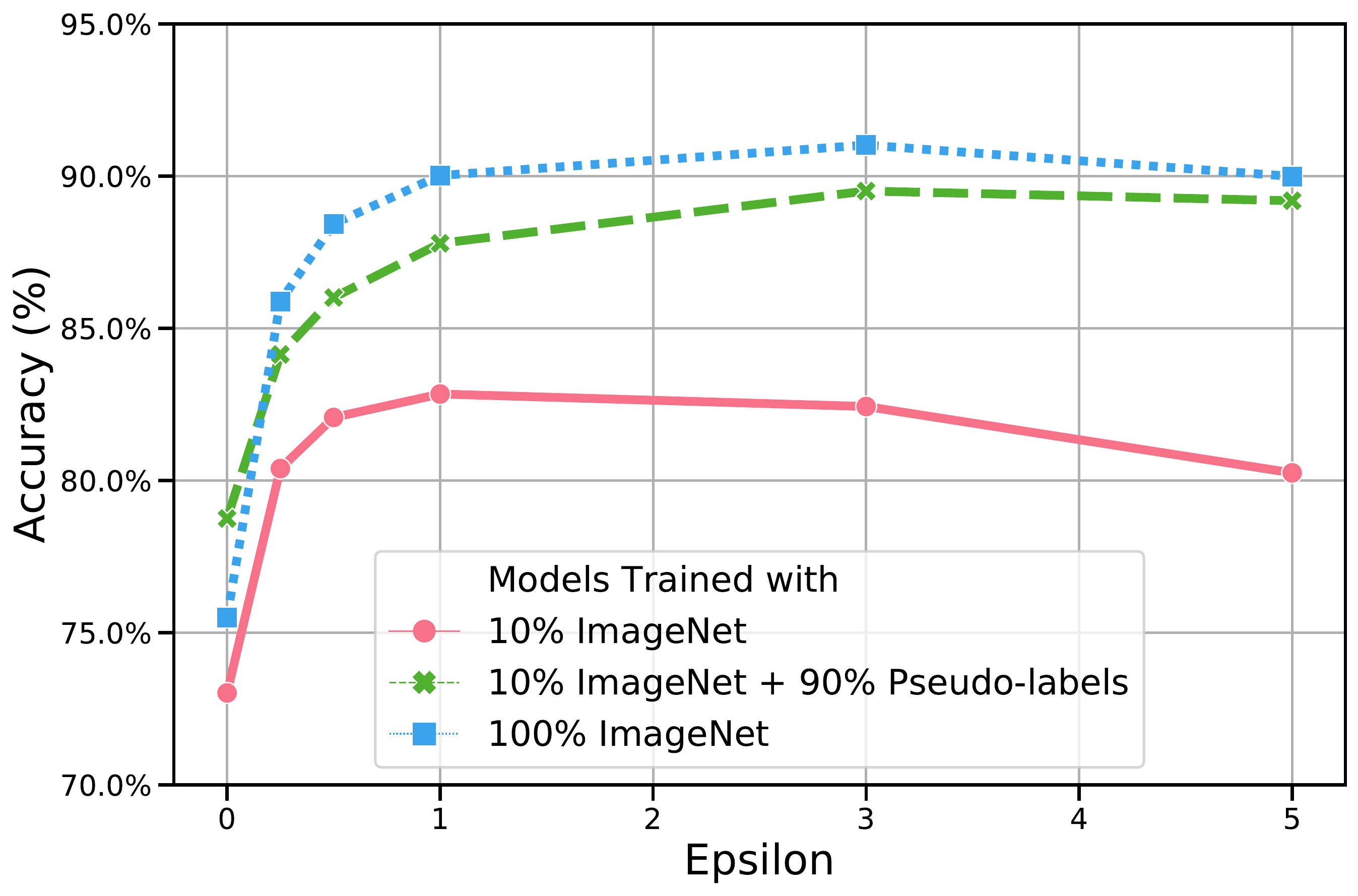} &   \includegraphics[width=50mm]{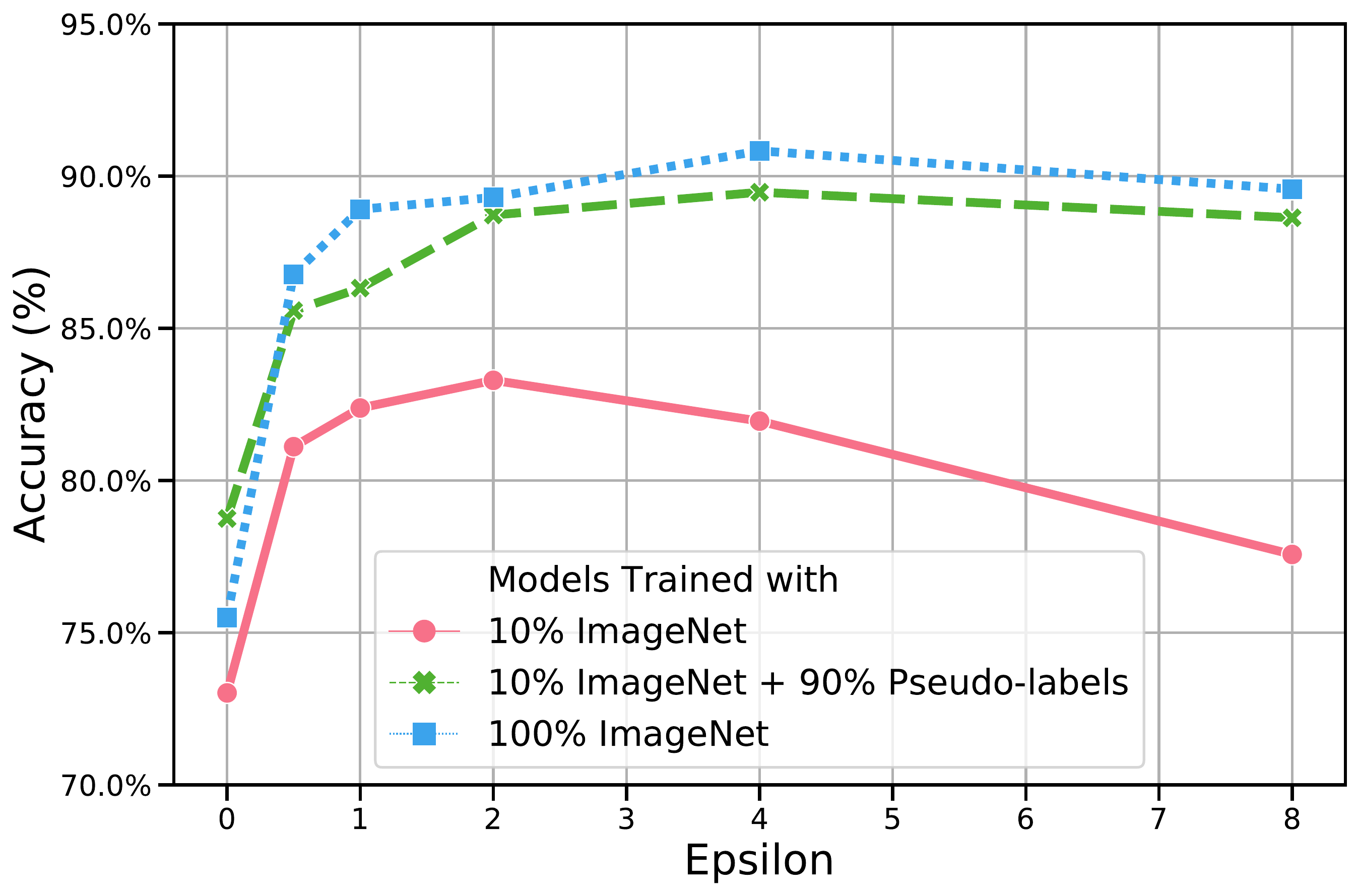} \\
(a) $\ell_{2}$ norm training  & (b) $\ell_{\infty}$ norm training \\[6pt]
\end{tabular}
\caption{Comparison of target task (CIFAR-10) accuracy for models trained on source task (ImageNet) using (i). a 10\% sample of data from the source task, (ii). the 10\% sample with ground truth labels, and the remaining 90\% with pseudo-labels, and (iii). 100\% of the source task data with ground truth labels. Models trained on the source task with (a). $\ell_{2}$-adversarial training and (b). $\ell_{\infty}$-adversarial training both exhibit similar behavior. The x-axis refers to the magnitude, $\varepsilon$, used in adversarial training---larger values indicate allowing more difficult adversarial examples; $0$ corresponds to no adversarial training. The $\varepsilon$ value in (b) is scaled up by 255 so it corresponds to pixel difference on a [0,255] scale.}
\label{fig:experiment_across_epsilon}
\end{center}
\end{figure}
In Figure~\ref{fig:experiment_across_epsilon}, we plot the target task accuracy of models trained on our source task in 3 different settings, across different levels of adversarial training. Models in Figure~\ref{fig:experiment_across_epsilon} (a) and (b) are trained on the source task with $l_{2}$ and $l_{\infty}$-adversarial training respectively. We compare models trained on the source task using: 1) a fixed 10\% sample of ImageNet,; 2) the 10\% sample with ground truth labels and the remaining 90\% sample using the generated pseudo-labels; and 3) all of ImageNet with its ground truth labels.
 Adversarial training boosts transfer performance in all 3 settings. Pseudo-labels also boost transfer performance. In the $\varepsilon=0$ setting (i.e. no adversarial training), the highest target accuracy is obtained by using labeled examples with pseudo-labels (green points).  Moreover, adversarial training with pseudo-labels also increases performance; at the optimal setting for adversarial training, the difference from using pseudo-labels instead of ground truth labels is only $1.5\%$. 

\begin{table}[h]
\caption{Effect of amount of pseudo-labels on transfer task performance (measured with accuracy). At $0\%$, we just use $10\%$ of data from the source task; at $900\%$, we use all remaining $90\%$ of data with pseudo-labels (this is $9$ times the train set size). Adversarial training corresponds to using $\ell_{2}$-adversarial training with $\varepsilon=1$ on the source task. Results on additional datasets in Appendix.}
{\resizebox{\columnwidth}{!}{\begin{tabular}{@{}lccccccr@{}}
\toprule
Source Task  & Target Task & +0\% Pseudo-labels & +200\% Pseudo-labels & +500\% Pseudo-labels & +900\% Pseudo-labels \\ 
\midrule
ImageNet                       & CIFAR-10       & 73.0\% & 73.8\% & 77.1 \% & 78.8 \% \\
ImageNet (w/adv.training)                   & CIFAR-10       & 82.8\% & 85.7\% & 87.5 \% & 87.8 \% \\ \midrule
ImageNet                      & CIFAR-100      & 51.0\% & 52.9\% & 55.3 \% & 58.4\% \\
ImageNet (w/adv.training)                      & CIFAR-100      & 62.6\% & 65.2\% & 68.1 \% & 69.5 \% \\
\bottomrule
\end{tabular}
}}\label{tbl:pseudolabel_percent}
\end{table}

In Table~\ref{tbl:pseudolabel_percent} we investigate how the amount of pseudo-labeled data affects performance. We train models in 2 settings: with adversarial and non-adversarial (standard) training on the source task. The adversarial training corresponds to $\ell_{2}$ norm adversarial training with $\varepsilon = 1$. Across all settings, we observe that robust training improves performance, and adding more pseudo-labeled data improves performance with diminishing returns.

\paragraph{Theoretical illustration.}
We further support the above experimental observations with theories. We denote the unlabeled input data for each source task $t\in[T]$ as $X^{u}_t=\{x^{u,(t)}_i\}_{i=1}^{n^{u}_t}$. The algorithm we analyze is as the following:

\begin{algorithm}[tbh]
   \caption{Natural and Adversarial Learning for Linear Features with Pseudo-labeling}
   \label{alg:psuedo}
\textbf{Input:} $\{S_t\}_{t=1}^{T+1}$, $\{X^u_t\}_{t=1}^T$, $q$
\vspace{0.1cm}

\text{Step 1:} Train an initial classifier: $w^{(t)}_{init}=\argmin_{\|w\|\le 1} \frac{1}{n_t}\sum_{i=1}^{n_t}-y^{(t)}_i\langle w, x^{(t)}_i \rangle$

\text{Step 2:} Obtain pseudo labels:  $y^{u,(t)}_i=sgn(\langle w^{(t)}_{init}, x^{(t)}_i \rangle)$

\text{Step 3:} Obtain augmented data sets $S_{t,aug}$ by combining $S_t$ and $\{(x^{u,(t)}_i, y^{u,(t)}_i)\}_{i=1}^{n^u_t}$

\text{Step 4:} $(\hat{W}_{1,aug},\hat{w}^{(T+1)}_{2,aug})\leftarrow \text{Algorithm}~1(S_{t,aug},S_{T+1})$,

$\quad\quad \quad(\hat{W}^{adv}_{1,aug}\hat{w}^{adv,(T+1)}_{2,aug})\leftarrow \text{Algorithm}~2(S_{t,aug},S_{T+1},q)$

\textbf{Return} $\hat{W}_{1,aug},\hat{w}^{(T+1)}_{2,aug},\hat{W}^{adv}_{1,aug}\hat{w}^{adv,(T+1)}_{2,aug}$
\end{algorithm}
\begin{theorem}\label{thm:pl}
Denote $\tilde n =\min_{t\in[T]}n_t^u$ and assume $\tilde n>c_1\max\{pr^2/T,r^2\log(1/\delta)/T,r^2, n\}$ for some constant $c_1>0$. Assume $\sigma_r(M^\top M/T)=\Omega(1/r)$ and $n^{c_2}\gtrsim\tilde n\gtrsim n$ for some $c_2>1$, if $n\gtrsim (T+d)$ and $\min_{t\in[T]} \|a_t\|=\Theta(\log^2 n)$ and $\eta_i^{(t)}\sim \cN_p(0,\rho_t^2 I^2)$ for $\rho_t=\Theta(1)$. Let $\hat{W}_{1,aug}$ obtained in Algorithm \ref{alg:psuedo}, with probability $1-O(n^{-100})$,
\begin{equation*}
\|\sin\Theta(\hat{W}_{1,aug},B)\|_F\lesssim r\left(\sqrt{\frac{1}{\tilde n}}+\sqrt{\frac{p}{\tilde nT}}+\sqrt{\frac{\log n}{\tilde nT}}\right).
\end{equation*}
\end{theorem}

Comparing the results above with Lemma~\ref{lm:representation}, we theoretically justify that by incorporating unlabeled data, we are able to learn a better representation when $\tilde n\gg n$. In the following, we show that adversarial training, together with the pseudo-labeling, can further boost this improvement. 

\begin{theorem}\label{thm:padv}
Under the same conditions as those in Theorem~\ref{thm:pl},

(a). For $\ell_2$ attack, under assumptions same to the those in Theorem~\ref{thm:l2adv}, and additionally $\tilde n>c_1\max\{r^2,r/\alpha_T\}\max\{p\log T,\log n/T,1\}$ for a  universal constant $c_1$, and choose $\varepsilon \in [\max_{t\in S_1}\|a_t\|+c_3\sqrt{p\log T/\tilde n},\min_{t\in S_2}\|a_t\|-c_3\sqrt{p\log T/\tilde n}]$, we then have with probability at least $1-O(n^{-100})$,
    \begin{equation}
\|\sin\Theta(\hat{W}^{adv}_{1,aug},B)\|_F \lesssim (\alpha_T)^{-1}\left(\sqrt{\frac{r^2}{\tilde n}}+\sqrt{\frac{pr^2}{\tilde nT}}+\sqrt{\frac{r^2\log(n)}{\tilde nT}}\right);
\end{equation}
(b). For $\ell_\infty$ attack, under assumptions  same to the those in Theorem~\ref{thm:linftyadv}, and additionally $\tilde n>C_1\cdot r^2\max\{s^2\log^2T/T,1\}$ for a  universal constant $C_1$. There exists a universal constant $c_2$, such that if we choose $\varepsilon>c_3\sqrt{\log p/\tilde n}$, with probability at least $1-O(n^{-100})-O(T^{-100})$,
    \begin{equation*}
\|{\sin\Theta(\hat{W}^{adv}_{1,aug},B)\|_F\lesssim r(\sqrt{\frac{1}{\tilde n}}+\sqrt{\frac{s^2}{\tilde nT}})\cdot \log(T+p).}
\end{equation*}

\end{theorem}

Similar to the interpretations of Theorems~\ref{thm:l2adv} and \ref{thm:linftyadv}, Theorem~\ref{thm:padv} suggests that adversarial training can boost the representation learning either (i). when the signal to noise ratio is varying ($\ell_2$ adversarial training helps in this case) and (ii). where there are many redundant features in classification ($\ell_\infty$ adversarial training helps in this case. Same to the analysis before, we can obtain similar upper bounds on the excess risks as those in Theorems~\ref{thm:l2adv} and \ref{thm:linftyadv}  by using Eq. \eqref{eq:excess}.

\section{Discussion}\label{sec:con}
In this paper, we provide the first theoretical framework to explain how adversarial training on the source data improves transfer learning. We show that adversarial training helps learning a more robust representation, and therefore boosts the predictive performance on the target task. Additionally, we extend our analysis to the semi-supervised setting and show that adversarial training, together with pseudo-labeling, have complementary benefits and can further improve the transfer. 

\textbf{Societal impacts and limitations}
Transfer learning helps learn a well-performed machine learning model with only a small amount of labeled data from the target task. Our work contributes to this field by providing insights into factors that improve transfer learning. A limitation of our work is that we have to make some standard assumptions on the data generative distribution when developing theories, which were also made in several other theory papers. 
While the model is simple, it captures the essence of the problem studied in the paper and is the first tractable framework to study how adversarial training helps fixed-feature transfer learning. The analysis here are already challenging and are supported by our experiments. 

\bibliography{reference}
\bibliographystyle{plainnat}
\newpage

\newpage
\appendix
\noindent\textbf{\Large Appendix}

\paragraph{Outline} We provide detailed proofs for all of our theories in Secs.~\ref{proof:standard} to \ref{proof:lowerbound}. Sec.~\ref{sec:additional_exp} provides multiple additional experiments demonstrating that pseudo-labeling improves transfer learning and that combining pseudo-labeling with adversarial training in the source further improves tranferability. Sec.~\ref{sec:exp_details} provides additional details about our experiments.  

Recall that in the main context, in Algorithm \ref{alg:natural}, we have $\hat{W}_1\leftarrow \text{top-$r$}~\text{SVD of}~[\hat{\beta}_1,\hat{\beta}_2,\cdots,\hat{\beta}_T].$ Specifically, we assign the columns of  $\hat{W}_1$ as the collection of the top-$r$ left singular vectors of  $[\hat{\beta}_1,\hat{\beta}_2,\cdots,\hat{\beta}_T].$

The rest of proofs are based on the above methodology.

\section{Proof of Lemma \ref{lm:representation}}\label{proof:standard}

Let us define $\hat{\mu_t}=\sum_{i=1}^{n_t} x^{(t)}_iy^{(t)}/n_t$ and $\mu_t=Ba_t$ for all $t\in[T+1]$.

Notice that 
$$\hat{J}=(\hat{\mu}_1/\|\hat{\mu}_1\|,\cdots,\hat{\mu}_T/\|\hat{\mu}_T\|)=(\hat{\mu}_1,\cdots,\hat{\mu}_T)\text{diag}(\|\hat{\mu}_1\|^{-1},\cdots,\|\hat{\mu}_T\|^{-1})$$
As a result, doing SVD for $\hat{J}$ to obtain left singular vectors is equivalent to doing SVD for $\hat{\Phi}=(\hat{\mu}_1,\cdots,\hat{\mu}_T)$ to obtain left singular vectors (up to an orthogonal matrix, meaning rotation of the space spanned by the singular vectors) since multiplying a diagonal matrix on the right does not affect the collection of left singular vectors. It further means doing SVD for $\hat{J}$ to obtain left singular vectors is equivalent to obtaining left singular vectors for $\hat{\Phi}=(\hat{\mu}_1,\cdots,\hat{\mu}_T)\text{diag}(\|\mu_1\|^{-1},\cdots,\|\mu_T\|^{-1})$ (up to an orthogonal matrix).

We mainly adopt the Davis-Kahan Theorem in \cite{yu2015useful}. We further denote $\Phi=(\mu_1,\cdots,\mu_T)\text{diag}(\|\mu_1\|^{-1},\cdots,\|\mu_T\|^{-1})$.
 \begin{lemma}[A variant of Davis–Kahan Theorem]\label{lm:dk}
Assume $\min\{T,p\}>r$. For simplicity, we denote $\hat{\sigma}_1\ge \hat{\sigma}_2\ge \cdots \ge \hat{\sigma}_r$ as the top largest $r$ singular value of $\hat\Phi$ and  $\sigma_1\ge \sigma_2\ge \cdots \ge \sigma_r$ as the top largest $r$ singular value of $\Phi$. Let $V=(v_1,\cdots,v_r)$ be the orthonormal matrix consists of left singular vectors corresponding to $\{\sigma_i\}_{i=1}^r$ and $\hat V=(\hat v_1,\cdots,\hat v_r)$ be the orthonormal matrix consists of left singular vectors corresponding to $\{\hat \sigma_i\}_{i=1}^r$. Then,
$$\|\sin\Theta(\hat{V},V)\|_F\lesssim \frac{(2\sigma_1+\|\hat{\Phi}-\Phi^*\|_{op})\min \{r^{0.5}\|\hat{\Phi}-\Phi^*\|_{op},\|\hat{\Phi}-\Phi^*\|_{F}\}}{\sigma^2_{r}}.$$
Moreover, there exists an orthogonal matrix $\hat{O}\in \bR^{r\times r}$, such that $\|\hat{V}\hat{O}-V\|_F\le\sqrt{2}\|\sin\Theta(\hat{V},V)\|_F$, and
$$\|\hat{V}\hat{O}-V\|_F\lesssim \frac{(2\sigma_1+\|\hat{\Phi}-\Phi^*\|_{op})\min \{r^{0.5}\|\hat{\Phi}-\Phi^*\|_{op},\|\hat{\Phi}-\Phi^*\|_{F}\}}{\sigma^2_{r}}.$$

\end{lemma}

It is worth noticing that actually $B$ plays the exact same role as $V$. Since $B$ has orthonormal columns, for $\phi$ we have
\begin{align*}
\Phi&=B(a_1,\cdots,a_T)\text{diag}(\|\mu_1\|^{-1},\cdots,\|\mu_T\|^{-1})\\
&=B(a_1,\cdots,a_T)\text{diag}(\|a_1\|^{-1},\cdots,\|a_T\|^{-1}).
\end{align*}
Thus, $B$ is a solution of the SVD step in Algorithm \ref{alg:natural}.
\begin{lemma}[\textbf{Restatement of Lemma \ref{lm:representation}}]
Under Assumption \ref{ass:1}, if  $n>c_1\max\{pr^2/T,r^2\log(1/\delta)/T,r^2\}$ for some universal constant $c_1>0$ and $2r\le\min\{p,T\}$, for all $t\in [T]$. For $\hat{W}_1$ obtained in Algorithm \ref{alg:natural}, with probability at least $1-O(n^{-100})$,
\begin{equation*}
\|\sin\Theta(\hat{W}_1,B)\|_F\lesssim 
r\left(\sqrt{\frac{1}{n}}+\sqrt{\frac{p}{nT}}+\sqrt{\frac{\log n}{nT}}\right).
\end{equation*}
\end{lemma}

\begin{proof}
By a direct application of Lemma \ref{lm:dk}, we can obtain 
$$\|\sin\Theta(\hat{W}_1,B)\| _F\lesssim \frac{(2\sigma_1+\|\hat{\Phi}-\Phi\|_{op})\min \{r^{0.5}\|\hat{\Phi}-\Phi\|_{op},\|\hat{\Phi}-\Phi\|_{F}\}}{\sigma^2_{r}}$$

Besides, we know that the left singular vectors of $\Phi$ are the same as the ones of $M=[a_1,\cdots,a_T]$ since $\Phi=BM \text{diag}(\|a_1\|^{-1},\cdots,\|a_T\|^{-1})$.






To estimate $\|\hat{\Phi}-\Phi\|_{op}=\sup_{v\in\bS^{p-1}}\|v^\top(\hat{\Phi}-\Phi)\|$, for any fixed $v\in \bS^{p-1}$, by standard chaining argument in Chapter 6 in \cite{wainwright2019high}, we know that 
$$\bP\left(\|v^\top(\hat{\Phi}-\Phi)\|\gtrsim \sqrt{\frac{T}{n}}+ \sqrt{\frac{\log(1/\delta)}{n}}\right)\le \delta$$

Then, we use chaining again for the $\psi_2$-process $\{v:\|v^\top(\hat{\Phi}-\Phi)\|\}$, we obtain 
$$\bP\left(\sup_{v\in\bS^{p-1}}\|v^\top(\hat{\Phi}-\Phi)\|\gtrsim \sqrt{\frac{p}{n}}+\sqrt{\frac{T}{n}}+ \sqrt{\frac{\log(1/\delta)}{n}}\right)\le \delta.$$
Besides, we know $ \sigma_r(M)=\sqrt{T/r}$  by assumption \ref{ass:1}, and we also have  $\sum_{i=1}^r\sigma^2(M)= T$, thus, we know that $\sigma_1(M)$ and $ \sigma_r(M)$ are both of order $\Theta(\sqrt{T/r})$
$$\|\sin\Theta(\hat{W}_1,B)\|_F\lesssim \frac{(\sqrt{\frac{p}{n}}+\sqrt{\frac{T}{n}}+ \sqrt{\frac{\log(1/\delta)}{n}}+\sqrt{T/r})\sqrt{r}(\sqrt{\frac{p}{n}}+\sqrt{\frac{T}{n}}+ \sqrt{\frac{\log(1/\delta)}{n}})}{T/r},$$

by simple calculation, we further have 
$$\|\sin\Theta(\hat{W}_1,B)\|_F\lesssim r\sqrt{r}(\frac{1}{n}+\frac{p}{nT}+\frac{\log(1/\delta)}{nT})+r(\sqrt{\frac{1}{n}}+\sqrt{\frac{p}{nT}}+\sqrt{\frac{\log(1/\delta)}{nT}}).$$

If we further have $n>r\max\{p/T,\log(1/\delta)/T,1\}$, we further have 

$$\|\sin\Theta(\hat{W}_1,B)\|_F\lesssim r(\sqrt{\frac{1}{n}}+\sqrt{\frac{p}{nT}}+\sqrt{\frac{\log(1/\delta)}{nT}}).$$

Plugging into $\delta=n^{-100}$, the proof is complete.

\end{proof}
\section{Proof of Corollary \ref{col:excessrisk}}



\begin{corollary}[\textbf{Restatement of Corollary \ref{col:excessrisk}}]
Under Assumption \ref{ass:1}, if $n>c_1\max\{pr^2/T,r^2\log(1/\delta)/T,r^2, r n_{T+1}\}$ for some universal constant $c_1>0$, $2r\le\min\{p,T\}$, 
then for $\hat{W}_1$ obtained in Algorithm \ref{alg:natural}, with probability at least $1-O(n^{-100})$,
\begin{equation*}
	\cR(\hat{W}_1,\hat{w}^{(T+1)}_2)\lesssim \sqrt{\frac{r+\log n}{n_{T+1}}}+\sqrt{\frac{r^2p}{nT}}. 
\end{equation*} 
\end{corollary}
\begin{proof} By DK-lemma, we know there exists a $W^*_1$ such that $W^*_1\in \argmin_{W\in \bO_{p\times r}}\|W^{\top}\mu_{T+1}\|$ (the minimizer is not unique, so we use $\in$ instead of $=$ to indicate $W^*_1$ belongs to the set consists of minimizers) and  $\|W_1^*-\hat{W}_1\|$ is small.

\begin{align*}
\cR(\hat{W}_1,\hat{w}^{(T+1)}_2)&=L(\cP_{x,y}^{(T+1)},\hat{w}^{(T+1)}_2,\hat{W}_1)-\min_{\|w_2\|\le 1,W_1\in \bO_{p\times r}}L(\cP_{x,y}^{(T+1)},w_2,W_1)\\
&=-\langle \frac{\hat{W}^\top_1\hat{\mu}_{T+1}}{\|\hat{W}^\top_1\hat{\mu}_{T+1}\|},\hat{W}_1^\top\mu_{T+1} \rangle +\|W^{*\top}_1\mu_{T+1}\|\\
&=-\langle \frac{\hat{W}^\top_1\hat{\mu}_{T+1}}{\|\hat{W}^\top_1\hat{\mu}_{T+1}\|},\hat{W}_1^\top\mu_{T+1} \rangle+\langle \frac{W^{*\top}_1\hat{\mu}_{T+1}}{\|W^{*\top}_1\hat{\mu}_{T+1}\|},W^{*\top}_1\mu_{T+1} \rangle\\
&-\langle \frac{W^{*\top}_1\hat{\mu}_{T+1}}{\|W^{*\top}_1\hat{\mu}_{T+1}\|},W^{*\top}_1\mu_{T+1} \rangle+\|W^{*\top}_1\mu_{T+1}\|\\
&\lesssim \|\hat{W}_1-W^*_1\|\|\mu_{T+1}\|+\|W^{*\top}_1\mu_{T+1}-W^{*\top}_1\hat{\mu}_{T+1}\| \\
&\lesssim \|\hat{W}_1-W^*_1\|\|\mu_{T+1}\|+\|B^{\top}\mu_{T+1}-B^{\top}\hat{\mu}_{T+1}\|
\end{align*}
if  $n>r^2\max\{p/T,\log(1/\delta)/T,1\}$. The last formula is due to the fact that  $W^*_1$ and $B$ are different only up to an orthogonal matrix.

By standard chaining techniques, we have with probability $1-\delta$

$$\|B^{\top}_1\mu_{T+1}-B^{\top}_1\hat{\mu}_{T+1}\|\lesssim \sqrt{\frac{r}{n_{T+1}}}+\sqrt{\frac{\log(1/\delta)}{n_{T+1}}}.$$

Thus, we can further bound $\|\hat{W}_1-W^*_1\|$ by $\sqrt{2}\|\sin\Theta(\hat{W}_1,B)\|_F$, thus, by Lemma \ref{lm:representation}, we have 
\begin{equation*}
\cR(\hat{W}_1,\hat{w}^{(T+1)}_2)\lesssim \sqrt{\frac{r+\log(1/\delta)}{n_{T+1}}}+r(\sqrt{\frac{1}{n}}+\sqrt{\frac{p}{nT}}+\sqrt{\frac{\log(1/\delta)}{nT}}).
\end{equation*} 
Now, if we further have $n>r n_{T+1}$, we have 
\begin{equation*}
\cR(\hat{W}_1,\hat{w}^{(T+1)}_2)\lesssim \sqrt{\frac{r+\log(1/\delta)}{n_{T+1}}}+\sqrt{\frac{r^2p}{nT}}.
\end{equation*} 
\end{proof}
Plugging into $\delta=n^{-100}$, the proof is complete.
\section{Proof of Theorem  \ref{thm:l2adv}}\label{proof:thm1}
\begin{theorem}[\textbf{Restatement of Theorem \ref{thm:l2adv}}]
Under Assumption \ref{ass:2} and \ref{ass:3}, for $\|a_{T+1}\|=\alpha=\Omega(1)$, if $n>c_1\max\{r^2,r/\alpha_T\}\cdot\max\{p\log T,\log n/T,1\}$ and $n>c_2(\alpha\alpha_T)^2r n_{T+1}$  for  universal constants $c_1, c_2$, $2r\le\min\{p,T\}$. There exists a universal constant $c_3$, such that if we choose $\varepsilon \in [\max_{t\in S_1}\|a_t\|+c_3\sqrt{p\log T/n},\min_{t\in S_2}\|a_t\|-c_3\sqrt{p\log T/n}]$ (this set will not be empty if $T, n$ are large enough), for $\hat{W}^{adv}_1$, $\hat{w}^{adv,(T+1)}_2$ obtained in Algorithm \ref{alg:adv} with $q=2$, with probability at least $1-O(n^{-100})$, 
\begin{equation*}
	\|\sin\Theta(\hat{W}^{adv}_1,B)\|_F\lesssim (\alpha_T)^{-1}\left(\sqrt{\frac{r^2}{n}}+\sqrt{\frac{pr^2}{nT}}+\sqrt{\frac{r^2\log n}{nT}}\right),
\end{equation*}
and the excess risk 
\begin{equation*}
	\cR(\hat{W}^{adv}_1,\hat{w}^{adv,(T+1)}_2)\lesssim \alpha\sqrt{\frac{r+\log n}{n_{T+1}}}+(\alpha_T)^{-1}\left(\sqrt{\frac{r^2p}{nT}}\right).
\end{equation*} 
\end{theorem}
\begin{proof}
For $\ell_2$-adversarial training, we have 
\begin{align*}
\hat{\beta}^{adv}_t& =\argmin_{\|\beta_t\|\le 1}\max_{\|\delta_i\|_p\le\varepsilon} \frac{1}{n_t}\sum_{i=1}^{n_t}-y^{(t)}_i\langle \beta_t, x^{(t)}_i+\delta_i \rangle\\
&=\argmin_{\|\beta_t\|\le 1}\max_{\|\delta_i\|_p\le\varepsilon} \frac{1}{n_t}\sum_{i=1}^{n_t}-y^{(t)}_i\langle \beta_t, x^{(t)}_i \rangle+\varepsilon \|\beta_t\|
\end{align*}

Recall $\hat{\mu}_{t}=\frac{1}{n_t}\sum_{i=1}^{n_t}y^{(t)}_ix^{(t)}_i$, if we have $\|\hat{\mu}_t\|\ge \varepsilon$, then $\hat{\beta}^{adv}_t=\hat{\mu_t}/\|\hat{\mu_t}\|$, otherwise, $\hat{\beta}^{adv}_t=0$.

We denote $$\hat{G}=[\hat{\beta}^{adv}_1,\cdots,\hat{\beta}^{adv}_T].$$

Since $|S_1|=\Theta(T)$, there exists a universal constant $c_3$ such that with probability $1-\delta$, we have for all $i\in S_1$, $\hat{\mu}_i\le \|a_i\|+c_3\sqrt{p\log T/n}$. Thus, if $T$ is large enough, the set $[\max_{t\in S_1}\|a_t\|+c_3\sqrt{p\log T/n},\min_{t\in S_2}\|a_t\|-c_3\sqrt{p\log T/n}]$ is non-empty. If we choose $\varepsilon \in [\max_{t\in S_1}\|a_t\|+c_3\sqrt{p\log T/n},\min_{t\in S_2}\|a_t\|-c_3\sqrt{p\log T/n}]$, for all $t\in S_2$, $\hat{\beta}^{adv}_t=\hat{\mu_t}/\|\hat{\mu_t}\|$.  Meanwhile, $\hat{G}_{S_1}$ is a zero matrix. 

Notice that the left singular vectors obtained by applying SVD  to $\hat{G}$ for left singular vectors  is equivalent to applying SVD for left singular vectors  to $\hat{G}_{S_2}$, which is further equivalent to applying SVD for left singular vectors to $\hat{\Phi}_{S_2}$, given that $\hat{G}_2$ is equal to $\hat{\Phi}_{S_2}$ times a diagonal matrix on the right. Thus, we have

\begin{equation*}
\|\sin\Theta(\hat{W}^{adv}_1,B)\|_F \lesssim \frac{(2\sigma_1(\Phi_{S_2})+\|\hat{\Phi}_{S_2}-\Phi_{S_2}\|_{op})\min \{r^{0.5}\|\hat{\Phi}_{S_2}-\Phi_{S_2}\|_{op},\|\hat{\Phi}_{S_2}-\Phi_{S_2}\|_{F}\}}{\sigma^2_{r}(\Phi_{S_2})}.
\end{equation*}

By our assumptions, we know that 
$$\bP\left(\sup_{v\in\bS^{p-1}}\|v^\top(\hat{\Phi}_{S_2}-\Phi_{S_2})\|\gtrsim \alpha_T^{-1}( \sqrt{\frac{p}{n}}+\sqrt{\frac{T}{n}}+ \sqrt{\frac{\log(1/\delta)}{n}})\right)\le \delta.$$

As a result,
$$\|\sin\Theta(\hat{W}_1,B)\|_F\lesssim \alpha_T^{-2}r\sqrt{r}(\frac{1}{n}+\frac{p}{nT}+\frac{\log(1/\delta)}{nT})+\alpha_T^{-1}r(\sqrt{\frac{1}{n}}+\sqrt{\frac{p}{nT}}+\sqrt{\frac{\log(1/\delta)}{nT}}).$$

If we further have $n>\frac{r}{\alpha_T}\max\{p/T,\log(1/\delta)/T,1\}$, we further have 

$$\|\sin\Theta(\hat{W}_1,B)\|_F\lesssim (\alpha_T)^{-1}r(\sqrt{\frac{1}{n}}+\sqrt{\frac{p}{nT}}+\sqrt{\frac{\log(1/\delta)}{nT}}).$$

Now, if we further have $n>(\alpha\alpha_T)^2 r n_{T+1}$, we have 
\begin{equation*}
\cR(\hat{W}_1,\hat{w}^{(T+1)}_2)\lesssim \alpha\sqrt{\frac{r+\log(1/\delta)}{n_{T+1}}}+(\alpha_T)^{-1}\sqrt{\frac{r^2p}{nT}}.
\end{equation*} 
Plugging into $\delta=n^{-100}$, the proof is complete.

\end{proof}

\begin{remark}[$\ell_2$-adversarial training v.s. standard training]
The proof of the counterpart of Lemma \ref{lm:representation} under the setting of Theorem \ref{thm:l2adv} basically folllows similar methods in the proof of Lemma \ref{lm:representation}. The only modification is that we need an extra step:
\begin{align*}
\bP\left(\sup_{v\in\bS^{p-1}}\|v^\top(\hat{\Phi}-\Phi)\|\gtrsim \sqrt{\frac{p}{n}}+\sqrt{\frac{T}{n}}+ \sqrt{\frac{\log(1/\delta)}{n}}\right)&\le \bP\left(\sup_{v\in\bS^{p-1}}\|v^\top(\hat{\Phi}_{S_1}-\Phi_{S_1})\|\gtrsim \sqrt{\frac{p}{n}}+\sqrt{\frac{T}{n}}+ \sqrt{\frac{\log(1/\delta)}{n}}\right)\\
&+\bP\left(\sup_{v\in\bS^{p-1}}\|v^\top(\hat{\Phi}_{S_2}-\Phi_{S_2})\|\gtrsim \sqrt{\frac{p}{n}}+\sqrt{\frac{T}{n}}+ \sqrt{\frac{\log(1/\delta)}{n}}\right)
\end{align*}
and recall that both $|S_1|$ and $|S_2|$ are of order $\Theta(T)$.
\end{remark}

\section{Proof of Theorem  \ref{thm:linftyadv}}\label{proof:thm2}
\begin{theorem}[\textbf{Restatement of Theorem \ref{thm:linftyadv}}]
Under Assumptions \ref{ass:1} and \ref{ass:4}, if $n>c_1\cdot r^2\max\{s^2\log^2T/T,r n_{T+1}, 1\}$  for some universal constants $c_1>0$, $2r\le\min\{p,T\}$. There exists a universal constant $c_2$, such that if we choose $\varepsilon>c_2\sqrt{\log p/n}$, for and $\hat{W}^{adv}_1$, $\hat{w}^{adv,(T+1)}_2$ obtained in Algorithm \ref{alg:adv} with $q=\infty$, with probability at least $1-O(n^{-100})-O(T^{-100})$,
\begin{equation*}
	\|{\sin\Theta(\hat{W}^{adv}_1,B)\|_F\lesssim r\left(\sqrt{\frac{1}{n}}+\sqrt{\frac{s^2}{nT}}\right)\cdot \log(T+p),}
\end{equation*}
and the excess risk 
\begin{equation}
	{\cR(\hat{W}^{adv}_1,\hat{w}^{adv,(T+1)}_2)\lesssim\left(\sqrt{\frac{r+\log n}{n_{T+1}}}+r\sqrt{\frac{s^2}{nT}}\right)\cdot \log(T+p)}.
\end{equation} 
\end{theorem}
\begin{proof}
For $\ell_\infty$-adversarial training, we have 
\begin{align*}
\hat{\beta}^{adv}_t& =\argmin_{\|\beta_t\|\le 1}\max_{\|\delta_i\|_\infty\le\varepsilon} \frac{1}{n_t}\sum_{i=1}^{n_t}-y^{(t)}_i\langle \beta_t, x^{(t)}_i+\delta_i \rangle\\
&=\argmin_{\|\beta_t\|\le 1} \frac{1}{n_t}\sum_{i=1}^{n_t}-y^{(t)}_i\langle \beta_t, x^{(t)}_i \rangle+\varepsilon \|\beta_t\|_1\\
&=\argmin_{\|\beta_t\|\le 1}\langle \beta_t, \frac{1}{n_t}\sum_{i=1}^{n_t}-y^{(t)}_i x^{(t)}_i \rangle+\varepsilon \|\beta_t\|_1
\end{align*}

Recall $\hat{\mu}_{t}=\frac{1}{n_t}\sum_{i=1}^{n_t}y^{(t)}_ix^{(t)}_i$. By observation, when reaching minimum, we have to have $sgn(\beta_{tj})=sgn(\hat\mu_{tj})$, therefore
\begin{align*}
&\argmax_{\|\beta_t\|=1}\sum_{j=1}^d\hat\mu_{tj}\beta_{tj}-\varepsilon|\beta_{tj}|\\
=&\argmax_{\|\beta_t\|=1}\sum_{j=1}^d(\hat\mu_{tj}-\varepsilon\cdot sgn(\hat\mu_{tj}))\beta_{tj}\\
=&\frac{T_\varepsilon(\hat\mu)}{\|T_\varepsilon(\hat\mu)\|},
\end{align*}
where $T_\varepsilon(\hat\mu)$ is the hard-thresholding operator with $(T_\varepsilon(\hat\mu))_j=sgn(\hat\mu_j)\cdot\max\{ |\hat\mu_j|-\varepsilon,0\}$.

We denote $$\hat{G}=[\hat{\beta}^{adv}_1,\cdots,\hat{\beta}^{adv}_T].$$

By the choice of $\varepsilon$, $ \varepsilon\gtrsim C\sqrt\frac{\log p}{n}$ for sufficiently large $C$, we have that the column sparsities of $\hat G$ is no larger than $s\log T$. As a result, the total number of non-zero elements in $\hat G$ is less than $O(Ts\log T)$ with probability at least $1-T^{-100}$.

Now we divide the rows of $\hat G$ by two parts: $[p]=A_1\cup A_2$, where $A_1$ consists of indices of rows whose  sparsity smaller than or equal to $s$, and $A_2$ consists of indices of rows whose  sparsity larger than $s$. 

Since the  number of non-zero elements in $\hat G$ is less than $Ts\log T$, we have $|A_2|\le T\log T$. Using the similar analysis as in the proof of Lemma 1, we have $$
\|\hat\Phi_{A_2}-\Phi_{A_2}\|\le \sqrt\frac{T\log T}{n}.
$$
For the rows in $A_1$, all of them has sparsity $\lesssim s$, so the maximum $\ell_1$ norm of these rows $$
\|\hat\Phi_{A_1}-\Phi_{A_1}\|_\infty=O_P(s\sqrt\frac{\log T}{n}).
$$

Similarly, the maximum $\ell_1$ norm of the columns in $\hat G_{A_1}$ satisfies $$
\|\hat\Phi_{A_1}-\Phi_{A_1}\|_1=O_P(s\sqrt\frac{\log p}{n}).
$$

Therefore, we have $$
\|\hat\Phi_{A_1}-\Phi_{A_1}\|\le\sqrt{\|\hat\Phi_{A_1}-\Phi_{A_1}^*\|_\infty \|\hat\Phi_{A_1}-\Phi_{A_1}\|_1}= O_P(s\sqrt\frac{\log p+\log T}{n}).
$$
Consequently, 
$$
\|\hat\Phi-\Phi\|\le \|\hat\Phi_{A_1}-\Phi_{A_1}\|+\|\hat\Phi_{A_2}-\Phi_{A_2}\|=O_P(s\sqrt\frac{\log p+\log T}{n})
$$


As a result, when $s\sqrt\frac{\log p+\log T}{n}\lesssim T/r$, applying Lemma~\ref{lm:dk}, we obtain
$$\|\sin\Theta(\hat{W}_1,B)\|_F\lesssim \sin\theta(\hat{W}^{adv}_1,B)\lesssim (\sqrt{\frac{r}{n}}+\sqrt{\frac{rs^2}{nT}})\cdot \log(T+p).$$



Now, if we further have $n>(\alpha\alpha_T)^2 n_{T+1}/\nu$, we have 
\begin{equation*}
\cR(\hat{W}_1,\hat{w}^{(T+1)}_2)\lesssim \sqrt{\frac{r+\log(1/\delta)}{n_{T+1}}}+\sqrt{\frac{rs^2}{nT}}\cdot \log(T+p). 
\end{equation*} 
\end{proof}
\begin{remark}[$\ell_\infty$-adversarial training v.s. standard training]
The proof of the counterpart of Lemma \ref{lm:representation} under the setting of Theorem \ref{thm:linftyadv}  follows exact the same method in the proof of Lemma \ref{lm:representation}. 
\end{remark}
\section{Proof of the case with pseudo-labeling}

\begin{theorem}[Restatement of Theorem~\ref{thm:pl}]
Denote $\tilde n =\min_{t\in[T]}n_t^u$ and assume $\tilde n>c_1\max\{pr^2/T,r^2\log(1/\delta)/T,r^2, n\}$ for some constant $c_1>0$. Assume $\sigma_r(M^\top M/T)=\Omega(1/r)$ and $n^{c_2}\gtrsim\tilde n\gtrsim n$ for some $c_2>1$, if $n\gtrsim (T+d)$ and $\min_{t\in[T]} \|a_t\|=\Theta(\log^2 n)$ and $\eta_i^{(t)}\sim \cN_p(0,\rho_t^2 I^2)$ for $\rho_t=\Theta(1)$. Let $\hat{W}_{1,aug}$ obtained in Algorithm \ref{alg:psuedo}, with probability $1-O(n^{-100})$,
\begin{equation*}
\|\sin\Theta(\hat{W}_{1,aug},B)\|_F\lesssim r\left(\sqrt{\frac{1}{\tilde n}}+\sqrt{\frac{p}{\tilde nT}}+\sqrt{\frac{\log n}{\tilde nT}}\right).
\end{equation*}
\end{theorem}
\begin{proof}
Let us first analyze the performance of pseudo-labeling algorithm in each individual task. In the following, we analyze the properties of $y_i^{u,(t)}$ and $\hat\mu_{final}^{(t)}=\frac{1}{n_t^u+n_t}\sum_{i=1}^{n_t^u+n_t}(\sum_{i=1}^{n^t_u}x_i^u y_i^u+\sum_{i=1}^{n_t}x_i^u y_i^u)$. Since $\tilde n\gtrsim n$ and we only care about the rate in the result. In the following, we derive the results for $\hat\mu_{final}^{(t)}=\frac{1}{n_t^u}\sum_{i=1}^{n_t^u+n_t}(\sum_{i=1}^{n^t_u}x_i^u y_i^u).$ Also, for the notational simplicity, we omit the index $t$ in the following analysis. 

We follow the similar analysis of \citet{carmon2019unlabeled} to study the property of $y_i^u$. Let $b_i$ be the indicator that the $i$-th pseudo-label is incorrect, so that $x_i^u\sim N((1-2b_i)y_i^u\mu,I):=(1-2b_i)y_i^u\mu+\varepsilon_i^u$. Then we can write $$\hat\mu_{final}=\gamma\mu+\tilde\delta,$$
where $\gamma=\frac{1}{n_u}\sum_{i=1}^{n_u}(1-2b_i)$ and $\tilde\delta=\frac{1}{n_u}\sum_{i=1}^{n_u} \varepsilon_i^uy_i^u$.

Let's write $y_i^u=sign(x_i^\top\hat\mu)$. 
Using the rotational invariance of Gaussian, without loss of generality, we choose the coordinate system where the first coordinate is in the direction of $\hat\mu$. 
Then $y_i^u=sign(x_i^\top\hat\mu)=sign(x_{i1})=sign(y_i^*\frac{\mu^\top\hat\mu}{\|\hat\mu\|}+\varepsilon^u_{i1})$ and are independent with $\varepsilon_{ij}^u$ $(j\ge 2)$. 

As a result, \begin{align*}
    \frac{1}{n_u}\sum_{i=1}^{n_u} \varepsilon_{ij}^u\cdot y_i^u \stackrel{d}{=} \frac{1}{n_u}\sum_{i=1}^{n_u}\varepsilon_{ij}^u, \quad \text{ for $j\ge 2$}.
\end{align*}
Now let's focus on $\frac{1}{n_u}\sum_{i=1}^{n_u} \varepsilon_{i1}^u\cdot y_i^u$.
Let $y_i^*=(1-2b_i)y_i^u$, we have \begin{align*}
   \frac{1}{n_u}\sum_{i=1}^{n_u} \varepsilon_{i1}^u\cdot y_i^u=\frac{1}{n_u}\sum_{i=1}^{n_u} \varepsilon_{i1}^u\cdot y_i^*+2\frac{1}{n_u}\sum_{i=1}^{n_u} \varepsilon_{i1}^u\cdot b_i\stackrel{d}{=}\frac{1}{n_u}\sum_{i=1}^{n_u} \varepsilon_{i1}^u+2\frac{1}{n_u}\sum_{i=1}^{n_u} \varepsilon_{i1}^u\cdot b_i.
\end{align*}
Since \begin{align*}
    (\frac{1}{n_u}\sum_{i=1}^{n_u} \varepsilon_{i1}^u\cdot b_i)^2\le (\frac{1}{n_u}\sum_{i=1}^{n_u} (\varepsilon_{i1}^u)^2)(\frac{1}{n_u}\sum_{i=1}^{n_u}  b_i^2)\lesssim \frac{1}{n_u}\sum_{i=1}^{n_u}  b_i^2=\frac{1}{n_u}\sum_{i=1}^{n_u}  b_i\lesssim\E[b_i]+\frac{1}{\sqrt{n_u}}\lesssim +\frac{1}{n}+ \frac{1}{\sqrt{n_u}},
\end{align*}
where the last inequality is due to the fact that \begin{align*}
    \E[b_i]=&\Prob(y_i^u\neq y_i^*)=\Prob(sign(y_i^*\frac{\mu^\top\hat\mu}{\|\hat\mu\|}+\varepsilon^u_{i1})\neq y_i^*)\\
    \le&\Prob(sign(y_i^*\frac{\mu^\top\hat\mu}{\|\hat\mu\|}+\varepsilon^u_{i1})\neq y_i^*\mid \frac{\mu^\top\hat\mu}{\|\hat\mu\|}>\frac{1}{2}\|\mu\|)+\Prob(\frac{\mu^\top\hat\mu}{\|\hat\mu\|}>\frac{1}{2}\|\mu\|)\\
    \lesssim& \exp^{-\|\mu\|/2}+\frac{1}{n^C}
\end{align*}
As a result, we have $$
\tilde\delta\stackrel{d}{=}\frac{1}{n_u}\sum_{i=1}^{n_u}\varepsilon_{i}^u+e,
$$
where $\|e\|_2\lesssim\frac{1}{\sqrt{n_u}}+\frac{1}{n^C}$.

Additionally, we have $\gamma=\frac{1}{n_u}\sum_{i=1}^{n_u}(1-2b_i)=1-\frac{2}{n_u}\sum_{i=1}^{n_u} b_i=1-O(\frac{1}{\sqrt{n_u}}+\frac{1}{n^C})$.

As a result, for each $t\in[T]$, we have $$
\hat\mu_t=\mu_t+\frac{1}{n_u}\sum_{i=1}^{n_u}\varepsilon_{i}^u+e',
$$
with $\|e'\|_2\lesssim\frac{1}{\sqrt{n_u}}+\frac{1}{n^{C'}}$ being a negligible term.

Since $e'$ is negligible, we can then follow the same proof as those in Section~\ref{proof:standard} by considering $\tilde\mu_t=\mu_t+\frac{1}{n_u}\sum_{i=1}^{n_u}\varepsilon_{i}^u$ and obtain the desired results.

Similarly, due to the negligibility of $e'$, we can prove Theorem~\ref{thm:padv} by following the exact same techniques in Sections~\ref{proof:thm1} and~\ref{proof:thm2}. 
\end{proof}
\section{Lower bound proof}\label{proof:lowerbound}

\begin{proposition}[Restatement of Proposition~\ref{prop:lowerbound}]
Let us consider the parameter space $\Xi=\{A\in\R^{p\times r},B\in\R^{p\times r}:\sigma_r(A^\top A/T)\gtrsim 1, B^\top B=I_r\}$. If $nT\gtrsim rp$, we then have
$$
\inf_{\hat W_1}\sup_{\Xi}\E\|\sin\Theta(B,\hat W_1)\|_F\gtrsim
\sqrt\frac{rp}{nT}.
$$
\end{proposition}

We first invoke the Fano's lemma. 

\begin{lemma}[\cite{tsybakov2008introduction}]
\label{fano}
Let $M\ge 0$ and $\mu_0, \mu_1, ... ,\mu_M \in \Theta$. For some constants $\alpha\in(0,1/8), \gamma>0$, and any classifier $\hat G$, if ${\rm KL}(\Prob_{\mu_i}, \Prob_{\mu_0})\le \alpha \log M$ for all $1\le i \le  M$, and $L(\mu_i,\mu_j)$ for all $0\le i\neq j\le M$, then$$
\inf_{\hat \mu} \sup_{i\in[M]}\E_{\mu_i}[L(\mu_i, \hat\mu)]\gtrsim\gamma.
$$
\end{lemma}

Now we take $B_0, B_1, ... ,B_M$ as the $\eta$-packing number of $O^{p\times r}$ with the $\sin \theta$ distance. 

Then according to \cite{pajor1998metric,tripuraneni2020provable}, we have $$
\log M\asymp rd\log(\frac{1}{\eta}).
$$
For any $i\in[M]$, we have $$
{\rm KL}(\Prob_{B_i}, \Prob_{B_0})=\sum_{t=1}^T n\|(B_i-B_0)a_t\|^2\le nT\eta^2.
$$
Let $\eta=\sqrt{\frac{rd}{nT}}$, we complete the proof.

\section{Additional Empirical Results}\label{sec:additional_exp}

We provide additional results on transfer performance with varied amounts of pseudo-labels in Table~\ref{tbl:pseudolabel_percent_extra}. Here, we train models with both adversarial (allowed maximum perturbations of $\varepsilon=1$ with respect to the $\ell_{2}$ norm) and non-adversarial (standard) training on ImageNet. The observed trend is the same as on the CIFAR-10 and CIFAR-100 tasks from Table~\ref{tbl:pseudolabel_percent} -- both using robust training and additional pseudo-labeled data improve performance.

\begin{table}[h]
\caption{Additional results extending Table~\ref{tbl:pseudolabel_percent}. Effect of amount of pseudo-labels on transfer task performance (measured with accuracy). At $0\%$, we just use $10\%$ of data from the source task; at $900\%$, we use all remaining $90\%$ of data with pseudo-labels (this is $9$ times the train set size). Adversarial training corresponds to using $\ell_{2}$-adversarial training with $\varepsilon=1$ on the source task. As per Section 7 of \cite{salman2020adversarially}, images in all datasetsare down-scaled to $32\times32$ before scaling back to $224\times224$.}
{\resizebox{\columnwidth}{!}{\begin{tabular}{@{}lccccccr@{}}
\toprule
Source Task  & Target Task & +0\% Pseudo-labels & +200\% Pseudo-labels & +500\% Pseudo-labels & +900\% Pseudo-labels \\ 
\midrule
ImageNet                       & Aircraft \cite{maji2013fine}        & 17.3\% & 17.6\% & 17.9\% & 19.9\% \\
ImageNet (w/adv.training)      & Aircraft        & 21.2\% & 20.9\% & 24.0\% & 24.5\% \\ \midrule
ImageNet                       & Flowers  \cite{nilsback2008automated}       & 60.7\% & 64.9\% & 65.4\% & 66.5\% \\
ImageNet (w/adv.training)      & Flowers         & 66.9\% & 68.1\% & 70.0\% & 70.1\% \\ \midrule
ImageNet                       & Food    \cite{bossard2014food}        & 33.7\% & 36.0\% & 36.7\% & 37.2\% \\
ImageNet (w/adv.training)      & Food            & 35.8\% & 37.5\% & 39.4\% & 40.8\% \\ \midrule
ImageNet                       & Pets    \cite{parkhi2012cats}        & 43.2\% & 44.9\% & 48.4\% & 49.0\% \\
ImageNet (w/adv.training)      & Pets            & 47.9\% & 53.1\% & 58.9\% & 59.6\% \\
\bottomrule
\end{tabular}
}}\label{tbl:pseudolabel_percent_extra}
\end{table}

\section{Experiment Details}\label{sec:exp_details}

\subsection{Training Hyperparameters}
All of our experiments use the ResNet-18 architecture. When transferring to the target task, we only update the final layer of the model. Our hyperparameter choices are identical to those used in \cite{salman2020adversarially}:
\begin{enumerate}
    \item ImageNet (source task) models are trained with SGD for 90 epochs with a momentum of $0.9$, weight decay of $1e-4$, and a batch size of $512$. The initial learning rate is set to $0.1$ and is updated every 30 epochs by a factor of $0.1$. The adversarial examples for adversarial training are generated using 3 steps with step size $\frac{2\varepsilon}{3}$.  
    \item Target task models are trained for 150 epochs with SGD with a momentum of $0.9$, weight decay of $5e-4$, and a batch size of 64. The initial learning rate is set to $0.01$ and is updated every 50 epochs by a factor of $0.1$. 
\end{enumerate}

Data augmentation is also identical to the methods used in \cite{salman2020adversarially}. As per Section~7 of \cite{salman2020adversarially}, we scale all our target task images down to size $32\times32$ before rescaling back to size $224\times224$. 

Experiments were run on a GPU cluster. A variety of NVIDIA GPUs were used, as allocated by the cluster. Training time for each source task model was around 2 days (less when using subsampled data) using 4 GPUs. Training time for each target task model was typically between 1-5 hours (depending on the dataset) using 1 GPU.

\subsection{Pseudo-label Generation}
When subsampling ImageNet (our source task), the sampled 10\% with ground truth labels preserves the class label distribution. This sample is fixed for all our experiments. All ImageNet pseudo-labels are generated by a model trained on this 10\% without any adversarial training. This model has a source task test accuracy (top-1) of $44.0\%$.

When training models with pseudo-labels, we preserve the class label distribution of the original training set (i.e., we add less pseudo-labels for those classes that have fewer examples in the entire training set).

\end{document}